\documentclass[12pt,letterpaper]{article}
\usepackage{fullpage}
\usepackage[top=2cm, bottom=2cm, left=2.5cm, right=2.5cm]{geometry}
\usepackage{amsmath,amsthm,amsfonts,amssymb,amscd}
\usepackage{lastpage}
\usepackage{enumerate}
\usepackage{fancyhdr}
\usepackage{mathrsfs}
\usepackage{xcolor}
\usepackage{graphicx}
\usepackage{listings}
\usepackage{hyperref}
\usepackage{bm}
\usepackage{commath}
\usepackage{booktabs}
\usepackage{authblk}
\usepackage{algorithm}
\usepackage{algorithmic}

\DeclareMathOperator{\Proj}{\mathrm{Proj}}

\newcommand{\R}{\mathbb{R}}

\newtheorem{assumption}{Assumption}

\newtheorem{lemma}{Lemma}
\newtheorem{definition}{Definition}
\newtheorem{remark}{Remark}
\newtheorem{theorem}{Theorem}

\allowdisplaybreaks

\title{\LARGE \bf
New Algorithms for Discrete-Time Parameter Estimation
}

\author[1]{Yingnan Cui}
\author[2]{Joseph E. Gaudio}
\author[1]{Anuradha M. Annaswamy}
\affil[1]{Massachusetts Institute of Technology}
\affil[2]{Aurora Flight Sciences}

\date{}
\begin{document}
\allowdisplaybreaks

\maketitle
\thispagestyle{empty}
\pagestyle{empty}

\begin{abstract}
    We propose two algorithms for discrete-time parameter estimation, one for time-varying parameters under persistent excitation (PE) condition, another for constant parameters under no PE condition. For the first algorithm, we show that in the presence of time-varying unknown parameters, the parameter estimation error converges uniformly to a compact set under conditions of persistent excitation, with the size of the compact set proportional to the time-variation of unknown parameters. Leveraging a projection operator, the second algorithm is shown to result in boundedness guarantees when the plant has constant unknown parameters. Simulations show better convergence results compared to recursive least squares (RLS) and comparable results to RLS with forgetting factor.
\end{abstract}

\section{Introduction}
\label{sec:intro}
A central task in adaptive control is the compensation for online parametric uncertainties in the processes of adaptive filtering, prediction and control\cite{Narendra2005, Lavretsky2013, Goodwin_1984, Ioannou1996, Landau11}. Over the past several decades, several discrete-time online estimation schemes have been designed, including conditions for parametric convergence\cite{Ljung_1987, Anderson_1982, ortega20}. Most of the methods in these works assume that the unknown parameters are constant. This paper focuses on the case where parameters are varying with time.

In many of the adaptive filtering, prediction, and control problems, the underlying structure consists of a linear regression relation between the prediction error and parameter error\cite{Ioannou1996, Aranovskiy_2019, Aranovskiy_2016, narendra11, Gaudio_2021a, gaudio20thesis}. A typical gradient-based algorithm, for example, minimizes a quadratic criterion in terms of the prediction error\cite{Goodwin_1984,Landau11}. In adaptive systems, a necessary and sufficient condition that is needed to ensure uniform asymptotic convergence of the estimates to their true values is noted as persistent excitation \cite{Boyd_1983, Boyd_1986}. The gradient descent method has a constant learning rate which often leads to extremely slow convergence of the estimates. The well-known recursive least squares (RLS) algorithm has a covariance matrix that is time-varying and is adjusted through the use of the outer product of the underlying regressor\cite{Goodwin_1984,Landau11}. One of the main inadequacies of the RLS algorithm is its inability to ensure estimation in the presence of time-varying parameters, which is due to the fact that the learning rates themselves can converge rapidly to zero and as a result, unable to track the changes in the parameters.
The literature on online parameter estimation algorithms in discrete-time is vast (see for example, \cite{Ljung_1987,Anderson_1982,Boyd_1986,ortega20,Boyd_1983, Aranovskiy_2019, Aranovskiy_2016, narendra11, bruce20, bruce20a, pal08, islam20, yu13, fort81}). In many of them, the learning rates are assumed to be constant. Of these, works such as \cite{bruce20, bruce20a, islam20, yu13, ortega20} have focused on a RLS-type algorithm and address variations thereof. In \cite{bruce20, bruce20a, pal08}, combined variable-rate and variable-direction forgetting is adopted to address some of the inadequacies of the RLS algorithm. In \cite{islam20}, a data-dependent forgetting rate is used in the update of the covariance matrix. A nonlinear weighting coefficient is devised in the update of the covariance matrix in \cite{yu13}. Another interesting approach has been suggested in \cite{ortega20}, which consists of a dynamic regressor extension and mixing (DREM) idea, and it tries to lift the limitation of the input signal by a linear, BIBO-stable operator. In most of these works, the unknown parameters are assumed to be constant. In addition, the focus in most of these papers is more in showing boundedness of the parameter estimates, rather than the speed of convergence of these estimates to their true values.

In this paper, we present two algorithms: the first assumes that persistent excitation conditions are met and the second assumes that they are not, both of which represent the two main contributions of this paper. Time-varying learning rate in the form of a time-varying gain matrix is included in both algorithms, but no covariance resetting or forgetting factor is introduced. It is shown that under persistent excitation, when the unknown parameter is time-varying, the first algorithm leads to convergence of the parameter error to a compact set that is proportional to the rate of variation of the unknown parameters. The convergence is uniform and asymptotic under persistent excitation. For the second algorithm, no assumptions of persistent excitation are made, and it is shown that when the unknown parameters are constants, the parameter estimation error is bounded. In contrast to RLS with forgetting based algorithms such as those in \cite{fort81, pal08, bruce20, bruce20a}, we provide a stability proof in this paper for this latter case when persistent excitation is absent, through the construction of a Lyapunov function. Simulation results show that our algorithms are comparable to RLS with forgetting, and result in much faster convergence compared to RLS algorithms.

Section \ref{sec:pre} presents the definitions of discrete-time persistent excitation and some mathematical preliminaries. The underlying problems are presented in Section \ref{sec:algorithm}. Section \ref{sec:algo} presents two classes of algorithms. Section \ref{sec:stab-conv-analys} gives stability and convergence analysis for the first algorithm. Section \ref{sec:proj} introduces projection in the update of the time-varying gain matrix and establishes stability for the second algorithm. Simulation results are shown in Section \ref{sec:sim}. Section \ref{sec:conclusion} provides a summary and concluding remarks.

\section{Preliminaries}
\label{sec:pre}
We denote $\mathbb{N}^+$ as the set of positive integers, and $\|\cdot\|$ as the Euclidean norm.
\begin{definition}[\hspace{1sp}\cite{lee88}]\label{d:PE}
  A bounded function $\phi_k:[k_0,\infty)\rightarrow\mathbb{R}^N$ is persistently exciting (PE) at a level $\alpha$  if there exists $\Delta T\hspace{-.045cm}\in\hspace{-.045cm}\mathbb{N}^+$ and $\alpha\hspace{-.045cm}>\hspace{-.045cm}0$ such that
  \begin{equation*}
    \sum_{i=k+1}^{k + \Delta T} \phi_i\phi_i^\top \succeq \alpha I, \;\; \forall k\geq k_0.
  \end{equation*}
\end{definition}


In Definition \ref{d:PE}, the degree of excitation is given by $\alpha$. It can be noted that this PE condition pertains to a property of $\phi$ that needs to hold over a moving window for all $k\geq k_0$. The following matrix inversion lemma is useful in deriving the results that follow.
\begin{lemma}
  \label{lemma:kailath}
  (Kailath Variant of the Woodbury Identity \cite{petersen2008matrix}.)
  \[(A + BC)^{-1} = A^{-1} - A^{-1}B(I + CA^{-1}B)^{-1}CA^{-1}\]
\end{lemma}

\section{Problem Formulation}
\label{sec:algorithm}
The class of discrete-time models we consider in this paper is of the form\cite{Goodwin_1984,Landau11}
\begin{equation}
  \label{eq:10}
    y_k = -\sum_{i=1}^{n}a_{i,k}^*y_{k-i} + \sum_{j=1}^{m}b_{j,k}^*u_{k-j-d} + \sum_{\ell=1}^pc^*_{\ell,k}f_\ell(y_{k-1}, \ldots, y_{k-n}, u_{k-1-d}, \ldots, u_{k-m-d}),
\end{equation}
where $a_{i,k}^*$, $b_{j,k}^*$ and $c_{\ell, k}^*$ are unknown parameters that may vary with time and need to be identified, and $d$ is a known time-delay. The function $f_\ell$ is an analytic function of its arguments and is assumed to be such that the system in \eqref{eq:10} is bounded-input-bounded-output (BIBO) stable. Denote $z_{k-1} = [y_{k-1}, \ldots, y_{k-n}]^\top$ and $v_{k-d-1} = [u_{k-1-d}, \ldots, u_{k-m-d}]^\top$. We rewrite \eqref{eq:10} in the form of a linear regression
\begin{equation}
  \label{eq:11}
  y_k = \phi_k^\top\theta_k^*,
\end{equation}
where $\phi_k = [z_{k-1}^\top, v_{k-d-1}^\top, f_1(z_{k-1}^\top, v_{k-d-1}^\top), \ldots, \allowbreak f_p(z_{k-1}^\top, v_{k-d-1}^\top)]^\top$ is a regressor determined by exogenous signals and $\theta_k^* = [a_{1,k}^*, \ldots, a_{n,k}^*, b_{1,k}^*, \ldots, b_{m,k}^*,c_{1,k}^*, \ldots, c_{\ell,k}^*]^\top$ is the underlying unknown parameter vector. We propose to identify the parameter $\theta^*_k$ as $\theta_k$ using an estimator
\begin{equation}
  \hat{y}_k = \phi_k^\top\theta_k.
  \label{eq:estimate}
\end{equation}
This leads to a prediction error
\begin{equation}
  \label{eq:12}
  e_{y,k} = \phi_k^{\top}\tilde{\theta}_k,
\end{equation}
where $e_{y,k} = \hat{y}_k - y_k$ is the output prediction error and the parameter error is $\tilde{\theta}_k = \theta_k - \theta_k^*$. We define the time variation in $\theta_k^*$ as $\Delta \theta^*_k = \theta_k^* - \theta_{k-1}^*$.

\begin{assumption}
  \label{ass:1}
The time variation $\Delta \theta_k^*$ is bounded, i.e.,
\begin{equation}
  \|\Delta\theta_k^*\| \leq \Delta^*.
  \label{eq:para-diff}
\end{equation}
\end{assumption}

\begin{assumption}
  \label{ass:2}
  There exists a $\theta^*_{\max}\geq 0$ such that
\begin{equation}
  \label{eq:para-bound}
  \|\theta_k^*\| \leq \theta_{\max}^*.
\end{equation}
\end{assumption}
In the following sections, we will consider two different cases: In the first case, there is persistent excitation in the input signal while Assumption \ref{ass:1} and Assumption \ref{ass:2} are assumed; In the second case, no persistent excitation is assumed and we only consider the simpler case where $\theta_k^*=\theta^*$, a constant.
\section{Adaptive Laws with a Time-Varying Learning Rate: The Algorithms}
\label{sec:algo}
\subsection{Algorithm 1: For Inputs with Persistent Excitation}
We propose the following recursive algorithm for estimating $\theta_k^*$ as $\theta_k$, for all $k\geq 0$:
\begin{align}
  \theta_{k+1} &= \theta_k -\frac{\lambda_\Gamma\kappa\Gamma_k\phi_ke_k}{1 + \|\phi_k\|^2}, \label{eq:3}\\  
  \Omega_k &= (1-\lambda_\Omega)\Omega_{k-1} + \frac{\phi_k\phi_k^\top}{1 + \|\phi_k\|^2}, \label{eq:1}\\
    \Gamma_k &= \Gamma_{k-1} + \lambda_\Gamma(\Gamma_{k-1} - \kappa\Gamma_{k-1}\Omega_k\Gamma_{k-1}), \label{eq:2}
\end{align}
where $\lambda_\Gamma$, $\kappa$ are positive scalars and $\Gamma_k \in\R^{N\times N}$ is the time-varying gain matrix. The initial value $\Gamma_{0}$ is a symmetric positive definite matrix chosen so that $\Gamma_{0} \preceq \Gamma_{\max}I$, where $\Gamma_{\max}\in\R^+$ is the maximum learning rate that is assumed to be given. The information matrix $\Omega_k\in\R^{N\times N}$ captures any excitation that may be present in the regressor $\phi_k$. The initial value of $\Omega_k$ is a symmetric positive definite matrix with $0 \prec \Omega_0 \preceq I$. The parameter $0 < \lambda_\Omega < 1$ is used for the adjustment of the evolution of $\Omega_k$ and is directly related to the upper bound of $\Omega_k$. The scalars $\lambda_\Gamma$, $\kappa$ and $\lambda_\Omega$ represent various weights of the proposed algorithm. 

The main innovation in the algorithm above is in the evolution of $\Gamma_k$. Unlike the recursive least squares algorithm with forgetting (see \cite{Goodwin_1984} for example), the algorithm in \eqref{eq:2} automatically ensures that $\Gamma_k$ stays within certain bounds, through a combination of linear and nonlinear components. As will be shown in the following section, stability and exponential decrease of parameter error towards a compact set are established under appropriate excitation conditions. In what follows, we assume that the regressors are persistently exciting, satisfying Definition \ref{d:PE}.
\subsection{Algorithm 2: For Inputs with No Excitation}
\setcounter{algorithm}{2}
When there is no persistent excitation in the input signal, we propose the following update of $\Gamma_k$ as a replacement of \eqref{eq:2} in \eqref{eq:3}-\eqref{eq:2}:
\begin{equation}
  \label{eq:39}
  \Gamma_k = \Proj_{\Gamma_{\max}}[\Gamma_{k-1} + \lambda_\Gamma(\Gamma_{k-1} - \kappa\Gamma_{k-1}\Omega_k\Gamma_{k-1})],
\end{equation}
where $\Proj(\cdot)$ is defined as in Algorithm \ref{alg:1}. The stability and convergence properties of these two algorithms are discussed in the following sections.
\begin{algorithm}
	\caption{The Projection Operator $\Gamma_k = \Proj(\Gamma_k')$}\label{alg:1}
	\begin{algorithmic}
		\REQUIRE $\Gamma_{\max}, \Gamma_k' > 0$
		\IF{$\|\Gamma_k'\|_2 > \Gamma_{\max}$}
			\STATE $UDU^\top\leftarrow \Gamma_k'$
			\FOR{$i\leftarrow 1:N$}
				\IF{$D_{i,i} > \Gamma_{\max}$}
					\STATE $D_{i,i}\leftarrow \Gamma_{\max}$
				\ENDIF
			\ENDFOR
			\STATE $\Gamma_k \leftarrow UDU^\top$
		\ELSE
			\STATE $\Gamma_k = \Gamma_k'$
		\ENDIF
	\end{algorithmic}
\end{algorithm}	

\section{Stability and Convergence Analysis for Algorithm 1}
\label{sec:stab-conv-analys}
Before addressing the boundedness of the parameter estimates and their convergence, we establish the boundedness of the information matrix $\Omega_k$ in \eqref{eq:1} and time-varying gain matrix $\Gamma_k$ defined in \eqref{eq:2}.

\subsection{Boundedness of $\Omega_k$}
\label{sec:boundedness-omega_k}
We first show that for any arbitrary regressor $\phi_k$, $\Omega_k$ is a positive semi-definite matrix with an upper bound.
\begin{lemma}
  \label{lemma:1}
  For the dynamic system in \eqref{eq:1}, if $0 < \lambda_\Omega < 1$ and $\Omega_0$ is such that $0\preceq\Omega_0\preceq I$, then $0\preceq \Omega_k\preceq \Omega_{\max}I$, for all $k \geq 1$, where $\Omega_{\max} = 1 / \lambda_\Omega$.
\end{lemma}
\begin{proof}
 We consider the lower bound and the upper bound of $\Omega_k$ separately.
 \begin{enumerate}
 \item First we prove $\Omega_k\succeq 0$ by induction. Let $v\in\R^N$ and $v\neq0$. Since $0<\lambda_\Omega < 1$ and $\Omega_0\succeq 0$,
   \begin{align*}
     v^\top\Omega_1v = (1 - \lambda_\Omega) v^\top\Omega_0v + \frac{\|\phi_1^\top v\|^2}{1 + \|\phi_1\|^2} \geq 0
   \end{align*}
   Therefore, $\Omega_1\succeq 0$. Similarly if $0\preceq\Omega_{k-1}\preceq \frac{1}{\lambda_\Omega} I$, it follows that
   \begin{equation*}
     v^\top\Omega_k v = (1 - \lambda_\Omega)v^\top\Omega_{k-1}v + \frac{\|\phi_k^\top v\|^2}{1 + \|\phi_{k}\|^2}\geq 0.
   \end{equation*}
 \item The lower bound $\Omega_k \succeq 0$ is immediate when applying the definition of positive-semidefinite matrices.
 \item Next we prove $\Omega_k\preceq \Omega_{\max}I$. Using convolution, $\Omega_k$ in \eqref{eq:1} can be written as
   \begin{align*}
     \Omega_k &= (1-\lambda_\Omega)^k\Omega_0 + \sum_{i=0}^{k-1}(1-\lambda_\Omega)^{k-1-i}\frac{\phi_i\phi_i^\top}{1 + \|\phi_i\|^2}\\
              &\preceq (1-\lambda_\Omega)^k \Omega_0 + I\sum_{i=0}^{k-1}(1-\lambda_\Omega)^i.
   \end{align*}
   Since $\Omega_0\leq I$, using simple algebra, we can then show that $\Omega_k\leq \Omega_{\max} I$, for all $k\geq 0$.
 \end{enumerate}
\end{proof}
Lemma \ref{lemma:1} guarantees that $\Omega_k$ in \eqref{eq:1} is a BIBO-stable system and determines an explicit upper bound of $\Omega_k$. In the next lemma, we focus on its lower bound, and show that $\Omega_k$ is positive definite, under persistent excitation (PE) conditions.
\begin{lemma}
  If $0 < \lambda_\Omega < 1$, and $\phi_k$ is persistently exciting over a period $\Delta T$ at a level $\alpha$ for all $k\geq k_1$, then $\Omega_k\succeq \Omega_{PE}I$, $\forall k\geq k_2$, where $\Omega_{PE} = (1 - \lambda_\Omega)^{\Delta T - 1}\alpha/d$, $d = \max_{i\geq k_1}\{1 + \|\phi_i\|^2\}$ and $k_2 = k_1 + \Delta T$.
  \label{lemma:3}
\end{lemma}
\begin{proof}
 Using convolution, $\Omega_k$ can be written as
 \begin{align*}
   \Omega_k &=(1-\lambda_\Omega)^{k}\Omega_0 + \sum_{i=0}^{k-1} (1-\lambda_\Omega)^{k-i-1}\frac{\phi_i\phi_i^\top}{1 + \|\phi_i\|^2}
 \end{align*}
 Since $\phi_k$ is persistently exciting,
 \begin{align*}
   \Omega_k 
   &\succeq (1-\lambda_\Omega)^{k-1}\sum_{i=k-\Delta T}^{k-1}(1-\lambda_\Omega)^{-i}\frac{\phi_i\phi_i^\top}{1 + \|\phi_i\|^2}\\
   &\succeq \Omega_{PE} I
 \end{align*}
\end{proof}
The effects of persistent excitation on the regressor, $\Omega_k$ and $\Gamma_k$ are shown in Table \ref{tab:prop2}.

\begin{table}[!htb]
  \centering
  \begin{tabular}{c c c c}
    \toprule[1.5pt]
    $k\in$& $[k_1, k_2]$& $k_2$ & $[k_2, \infty)$ \\
    \hline
    $\|\phi_k\phi_k^\top\|$ & $\geq 0$ & & \\
    $\|\sum_{k=k_1}^{k_2}\phi_k\phi_k^\top\|$ & & $\geq \alpha$ & \\
    $\|\Omega_k\|$ & & & $\geq \Omega_{PE}$\\
    $\|\Gamma_k\|$ & & & $\geq \Gamma_{PE}$\\
    \bottomrule[1.5pt]
  \end{tabular}
  \caption{Norm of Signals in Different Phases of PE Propagation}
  \label{tab:prop2}
\end{table}


\subsection{Boundedness of $\Gamma_k$}
\label{sec:boundedness-gamma_k}
One of the major contributions of this paper is the adjustment of the time-varying learning rate matrix $\Gamma_k$ in \eqref{eq:2}. The update in \eqref{eq:2} allows the linear part and the quadratic part to automatically balance each other so that $\Gamma_k$ is always bounded under conditions of persistent excitation. In this section we will show this result formally.

We first consider the scalar case of \eqref{eq:2}, given by
\begin{equation}
  \label{eq:4}
  \gamma_k = \gamma_{k-1} + \lambda_\Gamma(\gamma_{k-1} - \kappa\omega_k\gamma_{k-1}^2),
\end{equation}
where
\begin{equation}
0<\omega_{\min}\leq\omega_k\leq\omega_{\max}.\label{eq:omega_b}
\end{equation}
Let $\gamma_{\max} > 0$ and
\begin{equation}
  \gamma_{\min} = \min\left\{\frac{1}{\kappa\omega_{\max}}, \gamma_{\max} + \lambda_\Gamma(\gamma_{\max} - \kappa\omega_{\max}\gamma_{\max}^2)\right\}
  \label{eq:gammabound}
\end{equation}
where $\kappa_{\min} \leq \kappa < \kappa_{\max}$, and $\kappa_{\min}$, $\kappa_{\max}$ are defined as
\begin{equation}
  \kappa_{\min} = \frac{1}{\gamma_{\max}\omega_{\min}}, \;\; \kappa_{\max} = \frac{1 + \lambda_\Gamma}{\lambda_\Gamma\omega_{\max}\gamma_{\max}}.
  \label{eq:kappa}
\end{equation}
In the update of \eqref{eq:4}, $\gamma_{\max}$ and $\lambda_\Gamma$ are chosen arbitrarily by the user. The gain $\kappa$ is then chosen using the bounds in \eqref{eq:kappa}.
We now state the following lemma that establishes bounds  of $\gamma_k$ (See Fig. \ref{fig:parabola} and Fig. \ref{fig:parabola2} for an illustration of the bounds) which in turn depend on $\lambda_\Gamma, \gamma_{\max}$ and $\kappa$.
\begin{figure}[!htb]
  \centering
  \includegraphics[width=0.7\linewidth]{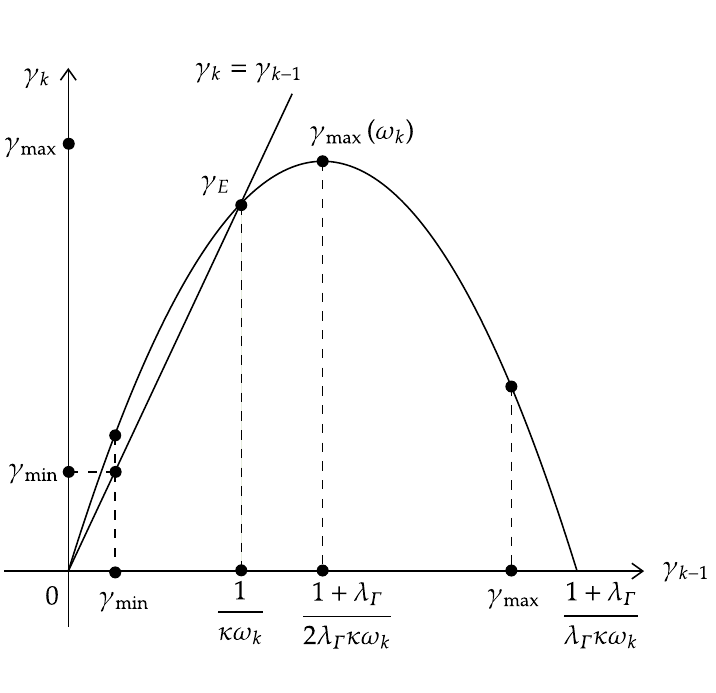}
  \caption{A parabola for illustrating the bounds of $\gamma_k$ whose update is in \eqref{eq:4} for the case when $\gamma_{\max}>\frac{1 + \lambda_\Gamma}{2\lambda_\Gamma\kappa\omega_k}$. For a given $\omega_k$, the line $\gamma_k = \gamma_{k-1}$ intersects the parabola at $\gamma_E$, giving $\gamma_k = \gamma_{k-1} = \frac{1}{\kappa\omega_k}$. When the parabola is above the straight line, $\gamma_k\geq\gamma_{k-1}$; Otherwise $\gamma_k \leq \gamma_{k-1}$. The point $\gamma_{\max}(\omega_k)$ is the maximum obtainable value for $\gamma_k$ and it will be no larger than $\gamma_{\max}$. When $\gamma_{k-1}=\gamma_{\max}$, the resulting $\gamma_k$ will be no less than $\gamma_{\min}$.}
  \label{fig:parabola}
\end{figure}
\begin{figure}[!htb]
  \centering
  \includegraphics[width=0.7\linewidth]{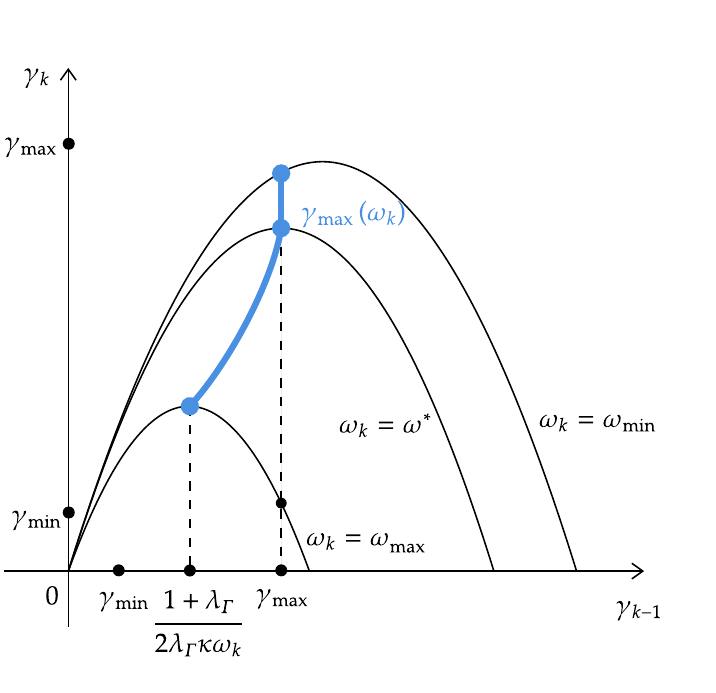}
  \caption{Illustration of the upper bounds of $\gamma_k$ as $\omega_k$ varies. The blue curve shows how $\gamma_{\max}(\omega_k)$ changes. The value of $\omega^*$ is defined such that \eqref{eq:omegaStar} is satisfied. Under the assumptions for Lemma \ref{lemma:bound_g}, $\gamma_{\max}(\omega_k)$ gets larger as $\omega_k$ decreases. The global maximum can be obtained when $\omega_k = \omega_{\min}$ and $\gamma_{k-1} = \gamma_{\max}$.}
  \label{fig:parabola2}
\end{figure}
\begin{lemma}
  If $0 < \lambda_\Gamma < \min\left\{\frac{1}{\omega_{\max}/\omega_{\min}-1}, 1\right\}$,  and  $\gamma_{\min} \leq \gamma_0 \leq \gamma_{\max}$, then the solutions of \eqref{eq:4} are such that $0< \gamma_{\min} \leq \gamma_k \leq \gamma_{\max}$, $\forall k\geq 0$.
  \label{lemma:bound_g}
\end{lemma}
\begin{proof}
 Defining the right hand side of Eq. \eqref{eq:4} as $f(\gamma_{k-1}, \omega_k)$, it is easy to see that $f(.,.)$ is quadratic in $\gamma_{k-1}$, with one maximum and two minima (see Fig. \ref{fig:parabola} and Fig. \ref{fig:parabola2}). Suppose $\omega^*$ is such that
 \begin{equation}
 \label{eq:omegaStar}
   \frac{1 + \lambda_\Gamma}{2\lambda_\Gamma\kappa\omega^*} = \gamma_{\max}.
 \end{equation}
 It is easy to check that $\omega^* \geq \omega_{\min}$. If $\omega_k\geq \omega^*$,
 \begin{align*}
   f(\gamma_{k-1},\omega_k) &\leq \gamma_{\max} + \lambda_\Gamma(\gamma_{\max} - \kappa\omega^*\gamma_{\max}^2)\\
   &\leq \gamma_{\max} + \lambda_\Gamma(\gamma_{\max} - \kappa\omega_{\min}\gamma_{\max}^2).
 \end{align*}
 If $\omega_k \leq \omega^*$,
 \begin{align*}
   f(\gamma_{k-1},\omega_k) &\leq \gamma_{\max} + \lambda_\Gamma(\gamma_{\max} - \kappa\omega_k\gamma_{\max}^2)\\
   &\leq \gamma_{\max} + \lambda_\Gamma(\gamma_{\max} - \kappa\omega_{\min}\gamma_{\max}^2).
 \end{align*}
 Noticing the bounds of $\lambda_\Gamma$ and $\kappa$, simple algebraic manipulations show that
 \begin{align*}
   \max_{\gamma_k,\omega_k}f(\gamma_{k-1},\omega_k) &= \gamma_{\max} + \lambda_\Gamma(\gamma_{\max} - \kappa\omega_{\min}\gamma_{\max}^2)\\
   &\leq \gamma_{\max}.
 \end{align*}
Similarly, it can be shown by inspection that 
\begin{equation*}
 \min_{\omega_k}f(\gamma_{k-1},\omega_k)=\gamma_{k-1} + \lambda_\Gamma(\gamma_{k-1} - \kappa\omega_{\max}\gamma_{k-1}^2).
\end{equation*}
Since the largest decrease of $f$ occurs when $\gamma_{k}$ is at its maximum value, it follows that $$\min_{\gamma_k,\omega_k}f(\gamma_{k-1},\omega_k)= \gamma_{\max} + \lambda_\Gamma(\gamma_{\max} - \kappa\omega_{\max}\gamma_{\max}^2).$$
 In addition, $\kappa\omega_{\max}>0$. Therefore, the lower bound of $\gamma_k$ and that it is positive are ensured through the choice of $\gamma_{\min}$ as in \eqref{eq:gammabound}.
\end{proof}
\begin{remark}
  It should also be noted that \eqref{eq:omega_b} implies that $\omega_k$ is persistently exciting, i.e., we consider the period $k\in[k_2, \infty)$ as shown in Table \ref{tab:prop2}.
\end{remark}
\begin{remark}
As we will show next, the properties we obtain in Lemma \ref{lemma:bound_g} for the scalar $\gamma_k$ will be directly utilized in showing the properties of the matrix $\Gamma_k$. Specifically, it will be shown that each eigenvalue of $\Gamma_k$ has an update similar to \eqref{eq:4}.
\end{remark}

We now proceed to the matrix case. First we consider the case when $\phi_k$ is persistently exciting and thus $\Omega_{PE}I \preceq \Omega_k\preceq \Omega_{\max}I$, $\forall k\geq k_2$, as shown in Lemma \ref{lemma:3}. Let $\Gamma_{\max}> 0$, $\kappa_{\min} = \frac{1}{\Gamma_{\max}\Omega_{PE}}$, $\kappa_{\max} = \frac{1+\lambda_\Gamma}{\lambda_\Gamma\Omega_{\max}\Gamma_{\max}}$, $\Gamma_{PE} = \min\left\{\frac{1}{\kappa\Omega_{\max}}, \Gamma_{\max} + \lambda_\Gamma(\Gamma_{\max} - \kappa\Omega_{\max}\Gamma_{\max}^2)\right\}$, the following Lemma gives the bounds on $\Gamma_k$ for $k \geq k_2$.
\begin{lemma}
  If $\phi_k$ is persistently exciting, $\Omega_{PE}I \preceq \Omega_k\preceq \Omega_{\max}I$, $\forall k\geq k_2$, learning rate gain $0 < \lambda_\Gamma < \min\left\{\frac{1}{\Omega_{\max}/\Omega_{PE}-1}, 1\right\}$, gain $\kappa_{\min}\leq\kappa < \kappa_{\max}$, and initial value $\Gamma_{PE}I \preceq \Gamma_{k_2} \preceq \Gamma_{\max}I$, then $\Gamma_{PE}I \preceq \Gamma_k \preceq \Gamma_{\max}I$, $\forall k > k_2$ for some $k_2 > k_1$.
  \label{lemma:5}
\end{lemma}
\begin{proof}
 We first establish the lower bound $\Gamma_{PE}$ using induction: Let $\Gamma_{PE}I \preceq \Gamma_{k-1} \preceq \Gamma_{\max}I$. It is easy to check that $\Gamma_{PE}> 0$. Since $\Gamma_{k-1}$ is positive definite, we can factorize $\Gamma_{k-1}$ as\cite{axler15}
 \begin{equation}
   \Gamma_{k-1} = Q_{k-1}\Lambda_{k-1}Q_{k-1}^\top
   \label{eq:18}
 \end{equation}
 where $\Lambda_{k-1}$ is diagonal and all of its entries are the eigenvalues of $\Gamma_{k-1}$, and $Q_{k-1}$ is orthogonal.
 Using Eq. \eqref{eq:2}  and the fact that $Q_{k-1}$ is orthogonal, we obtain that
 \begin{align*}
   \Gamma_{k} &= \Gamma_{k-1} + \lambda_\Gamma(\Gamma_{k-1} - \kappa\Gamma_{k-1}\Omega_k\Gamma_{k-1})\\
              &\succeq \Gamma_{k-1} + \lambda_\Gamma(\Gamma_{k-1} - \kappa\Omega_{\max}\Gamma_{k-1}^2)\\
              &= Q_{k-1}\left[\Lambda_{k-1} + \lambda_\Gamma(\Lambda_{k-1} - \kappa\Omega_{\max}\Lambda_{k-1}^2)\right]Q_{k-1}^\top,
 \end{align*}
 since
 \begin{equation*}
   x^\top(\Omega_{\max} I - \Omega_{k})x \geq 0, \; \forall x\neq 0.
 \end{equation*}
 Since $\Gamma_{k-1}$ is positive definite, it follows that
 \begin{equation*}
   v^\top\Gamma_{k-1}(\Omega_{\max} I - \Omega_{k})\Gamma_{k-1}v \geq 0, \; \forall v\neq 0.
 \end{equation*}
 Thus
 \begin{equation}
   \label{eq:5}
   \Gamma_{k-1}(\Omega_{\max}I - \Omega_{k})\Gamma_{k-1} \succeq 0.
 \end{equation}

 Denote the $i$th diagonal element of $\Gamma_{k-1}$ as $\lambda_{k-1,i}$, $i=1, \cdots, N$. Since $\Gamma_{PE} I \preceq \Gamma_{k-1} \preceq \Gamma_{\max} I$, we have $\Gamma_{PE}\leq\lambda_{k-1, i}\leq\Gamma_{\max}$, $\forall i$. Applying Lemma \ref{lemma:bound_g}, we obtain $\Gamma_{PE}\leq\lambda_{k, i}\leq\Gamma_{\max}$, $\forall i$ and the lower bound of $\Gamma_k$ is proved. Since the lower bound holds for $k=k_2$, proof by induction guarantees that the lower bound holds for all $k\geq k_2$.

To prove the upper bound, similar to Eq. \eqref{eq:5}, using Lemma \ref{lemma:3} we can show that
 \begin{equation}
   \label{eq:6}
   \Gamma_{k-1}(\Omega_{PE}I - \Omega_k)\Gamma_{k-1} \preceq 0.
 \end{equation}
 Therefore
 \begin{align*}
   \Gamma_{k} \preceq Q_{k-1}\left[\Lambda_{k-1} + \lambda_\Gamma(\Lambda_{k-1} - \kappa\Omega_{PE} \Lambda_{k-1}^2)\right]Q_{k-1}^\top.
 \end{align*}
Using the same procedure as above, it can be shown that the upper bound holds for $\Gamma_k$ if it holds for $\Gamma_{k-1}$, and using induction that the upper bound holds for all $k$, Lemma \ref{lemma:5} is proved.
\end{proof}
\begin{remark}
  The condition $\Omega_{PE} > \frac{1}{4}\Omega_{\max}$ in Lemma \ref{lemma:5} indicates that the minimum eigenvalue and the maximum eigenvalue should not be too far apart. This may seem to be restricted, but it should be noted that, when they are very far apart, there will be a similar problem of the disparage of convergence rates for methods like covariance resetting: Some parts of the covariance matrix are decreasing to zero quickly due to the large eigenvalues, while other parts will decrease slowly.
\end{remark}
\begin{remark}
Note that Lemma \ref{lemma:5} is essentially an extension of Lemma \ref{lemma:bound_g}. The extension to the matrix $\Gamma_k$ from scalar $\gamma_k$ has been carried out by diagonalization and applying Lemma \ref{lemma:bound_g}. In Lemma \ref{lemma:5}, it should be noted that $\Omega_{PE}$ is corresponding to $\omega_{\min}$ in Lemma \ref{lemma:bound_g}.
\end{remark}

\subsection{Stability and Parameter Convergence}
\label{sec:param-conv}
We now return to the main algorithm in \eqref{eq:3}, \eqref{eq:1} and \eqref{eq:2}. Note that in Section \ref{sec:boundedness-omega_k} and Section \ref{sec:boundedness-gamma_k} we have shown that both $\Omega_k$ and $\Gamma_k$ are bounded under persistent excitation. In this section we consider the boundedness of $\theta_k$ as well as its convergence.

We first introduce an intermediate variable in the evolution of $\Gamma_k$ in \eqref{eq:2} in the form of a matrix $\bar{\Gamma}_k$,
\begin{equation}
  \bar{\Gamma}_k = \Gamma_{k-1} - \lambda_\Gamma\kappa\Gamma_{k-1}\frac{\phi_{k-1}\phi_{k-1}^\top}{1 + \|\phi_{k-1}\|^2}\Gamma_{k-1},
  \label{eq:bar_gamma_pe}
\end{equation}
and consider a Lyapunov function candidate
\begin{equation}
  \label{eq:lya}
  V_{k} = \tilde{\theta}^\top_{k}\bar{\Gamma}_{k}^{-1}\tilde{\theta}_{k}.
\end{equation}
Under conditions of excitation, $V_k$ will be shown to be a Lyapunov function and converge to a compact set that scales with $\Delta^*$. 

In the case of persistent excitation, as proved in Lemma \ref{lemma:3}, $\Omega_{PE}I \preceq \Omega_k$, $\forall k\geq k_2$ for some $k_2 > 0$. Let $\lambda_\Omega$ satisfy
\begin{equation}
  \alpha\lambda_\Omega(1 - \lambda_\Omega)^{\Delta T - 1} > d,
  \label{eq:lOPE}
\end{equation}
where $\Delta T$, $\alpha$ and $d$ are defined in Lemma \ref{lemma:3}. Assume that $\lambda_\Gamma$ satisfies
\begin{equation}
  0 < \lambda_{\Gamma} < \min\left\{\frac{1}{\Omega_{\max}/\Omega_{PE}-1}, 1, \Omega_{PE}\right\}.
  \label{eq:lGPE}
\end{equation}
We define
\begin{align}
  \kappa_{\min} &= \frac{1}{\Gamma_{\max}\Omega_{PE}},\\
  \kappa_{\max} &= \min\Bigg\{\frac{1+\lambda_\Gamma}{\lambda_\Gamma\Omega_{\max}\Gamma_{\max}}, \frac{1}{\lambda_\Gamma\Gamma_{\max}},\allowbreak \frac{1}{(1-\lambda_\Omega)\Omega_{\max}\Gamma_{\max}}\Bigg\},
  \label{eq:kappaPE}
\end{align}
and choose any $\kappa$ that satisfies
\begin{equation}
  \kappa_{\min} \leq \kappa < \kappa_{\max}.
  \label{eq:kappaBound}
\end{equation}
According to Lemma \ref{lemma:5}, $\Gamma_{PE}I\preceq\Gamma_k\preceq\Gamma_{\max}I$, and $\Gamma_{PE}$ is defined as
\begin{equation}
  \Gamma_{PE} = \min\left\{\frac{1}{\kappa\Omega_{\max}}, \Gamma_{\max} + \lambda_\Gamma(\Gamma_{\max} - \kappa\Omega_{\max}\Gamma_{\max}^2)\right\}.
  \label{eq:GPE}
\end{equation}
Furthermore, let
\begin{equation}
  \bar{\Gamma}_{\min} = \min\left\{\Gamma_{\max} - \lambda_\Gamma\kappa\Gamma_{\max}^2, \Gamma_{PE} - \lambda_\Gamma\kappa\Gamma_{PE}^2\right\}.
  \label{eq:gamma_bar_pe}
\end{equation}
We give the expression of the upper-bound for the decay rate as
\begin{equation}
  \label{eq:mu1}
  \mu_1 = \Psi_{\min}\bar{\Gamma}_{\min},
\end{equation}
and consider a parameter $\mu_2$ that satisfies $0 < \mu_2 < \mu_1$. In \eqref{eq:mu1}, $\Psi_{\min} = (1 + \lambda_\Gamma \bar{\Gamma}_{\min}^{-1} \Gamma_{\max})^{-1}\allowbreak\Upsilon_{\min}\Gamma_{\max}^{-2}$, where $\allowbreak \Upsilon_{\min} = \min\big\{ \lambda_\Gamma\Gamma_{\max}\left[1 - \kappa(1 - \lambda_\Omega)\Omega_{\max}\Gamma_{\max}\right]$,  $\lambda_\Gamma\Gamma_{PE}\left[1 - \kappa(1 - \lambda_\Omega)\Omega_{\max}\Gamma_{PE}\right] \big\}$. Finally, a compact set is defined as
\begin{equation}
  \label{eq:14}
  D_{PE} = \left\{V | V\leq V_{PE} \right\},
\end{equation}
where
\begin{equation}
  \label{eq:13}
  V_{PE} = \left[\frac{\Delta^*\left(c_1 + \sqrt{c_1^2 + 4c_2(\mu_1 - \mu_2)}\right)}{2(\mu_1 - \mu_2)}\right]^2,
\end{equation}
\begin{equation}
  c_1 = \bar{\Gamma}_{\max}^{1/2} \bar{\Gamma}_{\min}^{-1},
  \label{eq:c1}
\end{equation}
\begin{equation}
  c_2 = \bar{\Gamma}_{\min}^{-1}.
  \label{eq:c2}
\end{equation}
We now state the stability result with persistent excitation.
\begin{theorem}
  \label{theo:pe}
  Let $\phi_k$ be persistently exciting. For the ARMA model in \eqref{eq:10}, under the assumptions in \ref{ass:1} and \ref{ass:2}, subject to the bounds in \eqref{eq:para-diff}, if the hyperparameters $\lambda_\Omega$, $\lambda_\Gamma$ and $\kappa$ satisfy \eqref{eq:lOPE}, \eqref{eq:lGPE} and \eqref{eq:kappaBound}, respectively, the adaptive algorithm in \eqref{eq:3}, \eqref{eq:1} and \eqref{eq:2} guarantees that in $D_{PE}^c$
  \begin{equation*}
    V_k \leq \exp\left[-\mu_2(k - k_2)\right]V_{k_2}, \;\forall k \geq k_2
  \end{equation*}
\end{theorem}
\begin{proof}
  Starting from \eqref{eq:3}, we get
  \begin{equation}
    \label{eq:7}
    \theta_{k} - \theta_{k-1}^*= \left(I - \frac{\lambda_\Gamma\kappa\Gamma_{k-1}\phi_{k-1}\phi_{k-1}^\top}{1 + \|\phi_{k-1}\|^2}\right)\tilde{\theta}_{k-1}.
  \end{equation}
  Using \eqref{eq:bar_gamma_pe}, $\Gamma_k$ can be written as
  \begin{equation}
    \Gamma_{k} = \bar{\Gamma}_{k} + \underbrace{\lambda_\Gamma\Gamma_{k-1} - \lambda_\Gamma\kappa(1-\lambda_\Omega)\Gamma_{k-1}\Omega_{k-1}\Gamma_{k-1}}_{\Upsilon_{k}}
    \label{eq:rest_v}
  \end{equation}
  Now we show that under the conditions in Theorem \ref{theo:pe}, $\bar{\Gamma}_{k}\succeq\bar{\Gamma}_{\min}I\succ 0$. With simple manipulations, it can be shown that
  \begin{equation}
    \bar{\Gamma}_{k} \succeq Q_{k-1}\left(\Lambda_{k-1} - \lambda_\Gamma\kappa\Lambda_{k-1}^2\right)Q_{k-1}^\top.
    \label{eq:15}
  \end{equation}
  It is easy to see that $\Gamma_{PE}\leq\lambda_{k-1, i}\leq\Gamma_{\max}$, $\forall i$, where $\lambda_{k-1,i}$ is the $i$th diagonal element of $\Gamma_{k-1}$. From the inequality in \eqref{eq:15}, $\bar{\Gamma}_{k}$ assumes its lower bound either when $\lambda_{k-1, i} = \Gamma_{PE}$ or when $\lambda_{k-1, i} = \Gamma_{\max}$. It follows therefore that $\bar{\Gamma}_k \succeq \bar{\Gamma}_{\min}I$, where $\bar{\Gamma}_{\min} > 0$ is defined in \eqref{eq:gamma_bar_pe}. It is also obvious that $\bar{\Gamma}_{k} \preceq \bar{\Gamma}_{\max}I$, where $\bar{\Gamma}_{\max} = \Gamma_{\max}$. Using the same approach, we can also show that $0 \prec \Upsilon_{\min}I \preceq \Upsilon_k \preceq \Upsilon_{\max}I$, where $\Upsilon_k$ is defined in \eqref{eq:rest_v} and $\Upsilon_{\max} = \lambda_\Gamma\Gamma_{\max}$.
  With these bounds established, we consider the Lyapunov function in \eqref{eq:lya}.
  Since $\tilde{\theta}_{k} = \theta_{k-1}^* - \theta_k^* + \bar{\Gamma}_{k}\Gamma_{k-1}^{-1}\tilde{\theta}_{k-1}$, we obtain that
  \begin{align*}
    V_{k} &= \left(\theta_{k-1}^* - \theta_k^* + \bar{\Gamma}_k\Gamma_{k-1}^{-1}\tilde{\theta}_{k-1}\right)^\top \bar{\Gamma}_k^{-1}\\
          &\hspace{0.5cm}\left[\theta_{k-1}^* - \theta_k^* + \left(I - \lambda_\Gamma\kappa\frac{\Gamma_{k-1}\phi_{k-1}\phi_{k-1}^\top}{1 + \|\phi_{k-1}\|^2}\right)\tilde{\theta}_{k-1} \right]\\
          &\leq \tilde{\theta}_{k-1}^\top \Gamma_{k-1}^{-1} \tilde{\theta}_{k-1} + 2\tilde{\theta}_{k-1}^\top\Gamma_{k-1}^{-1}(\theta_{k-1}^* - \theta_k^*) \\
          &\hspace{0.3cm}+ (\Delta^*)^2\bar{\Gamma}_{\min}^{-1}.
  \end{align*}
  Since $\Upsilon_{k}\succ 0$, $\Upsilon_k$ can be decomposed as $\Upsilon_{k} = \Upsilon_{k}^{1/2}\Upsilon_{k}^{1/2}$. Applying Lemma \ref{lemma:kailath} to \eqref{eq:rest_v}, we have
  \begin{equation}
    \label{eq:8}
    \Gamma_{k}^{-1} = \bar{\Gamma}_{k}^{-1} - \underbrace{\bar{\Gamma}_{k}^{-1}\Upsilon_{k}^{1/2}\left(I + \Upsilon_{k}^{1/2}\bar{\Gamma}_{k}^{-1}\Upsilon_{k}^{1/2}\right)^{-1}\Upsilon_{k}^{1/2}\bar{\Gamma}_{k}^{-1}}_{\Psi_{k}}.
  \end{equation}
  Since $\Upsilon_{k}^{1/2}\succ 0$ and $\bar{\Gamma}_{k}^{-1} \succ 0$, applying the definition of positive definite matrices we obtain that $\Upsilon_{k}^{1/2}\bar{\Gamma}_{k}^{-1}\Upsilon_{k}^{1/2} \succ 0$. Also, it follows that $\Upsilon_{k}^{1/2} \bar{\Gamma}_{k}^{-1} \Upsilon_{k}^{1/2} \preceq \lambda_\Gamma \bar{\Gamma}_{\min}^{-1} \Gamma_{\max}I$. Therefore $\left(I + \Upsilon_{k}^{1/2}\bar{\Gamma}_{k}^{-1}\Upsilon_{k}^{1/2}\right)^{-1} \succeq (1 + \lambda_\Gamma\bar{\Gamma}_{\min}^{-1}\Gamma_{\max})^{-1}I$. Applying the definition of positive definite matrices again, we get $\Psi_{k} \succeq \Psi_{\min}I \succ 0$. Substitute $\Gamma_k^{-1}$ in \eqref{eq:8} into $V_{k}$, we get
  \begin{align}
    \begin{split}
      V_{k} &\leq \tilde{\theta}_{k-1}^\top\Gamma_{k-1}^{-1}\tilde{\theta}_{k-1} + \Delta^*\bar{\Gamma}_{\max}^{1/2} \bar{\Gamma}_{\min}^{-1}\sqrt{V_{k-1}} + (\Delta^*)^2\bar{\Gamma}_{\min}^{-1} \\
      &\leq (1 - \underbrace{\Psi_{\min}\bar{\Gamma}_{\min}}_{\mu_1})V_{k-1} + \Delta^*\underbrace{\bar{\Gamma}_{\max}^{1/2} \bar{\Gamma}_{\min}^{-1}}_{c_1}\sqrt{V_{k-1}} \\
      &\hspace{0.3cm}+ (\Delta^*)^2\underbrace{\bar{\Gamma}_{\min}^{-1}}_{c_2},
    \end{split}
    \label{eq:vk_upper}
  \end{align}
  where we have replaced $\tilde{\theta}_{k-1}$ with an upper bound involving $V_{k-1}$. From \eqref{eq:8} and $\Gamma_k^{-1}\succ 0$, it follows that $\Psi_{\min}\bar{\Gamma}_{\min} < 1$. Therefore, the increment of $V_k$, $\Delta V_k$ can be expressed as
  \begin{align*}
    \Delta V_k &= V_k - V_{k-1}\\
               &\leq -\mu_1 V_{k-1} + c_1\Delta^*\sqrt{V_{k-1}} + c_2(\Delta^*)^2,
  \end{align*}
  where $\mu_1$ is defined in \eqref{eq:mu1}. Choosing any $0 < \mu_2 < \mu_1$, it can be noted that $\Delta V_k < 0$ in $D_{PE}^c$, where the compact set $D_{PE}$ is defined in \eqref{eq:14}.
  For any $k$ such that $V_k > V_{PE}$, it follows that
  \begin{equation*}
    V_k \leq (1 - \mu_2)V_{k-1}
  \end{equation*}
  We therefore obtain that when $k\geq k_2$, for all $V_k > V_{PE}$
  \begin{align*}
    V_k &\leq (1 - \mu_2)^{k-k_2}V_{k_2}\\
        &\leq \exp\left[-\mu_2(k - k_2)\right]V_{k_2}.
  \end{align*}
\end{proof}

\begin{remark}
  The key steps that enabled the proof of Theorem \ref{theo:pe} was 1) to express $\Gamma_k=\bar{\Gamma}_k + \Upsilon_k$, and 2) to define $V_k$ only using $\bar{\Gamma}_k$ as in \eqref{eq:bar_gamma_pe} rather than $\Gamma_k$. This provided the necessary tractability as it avoided matrix inversions involving $\Gamma_k$. From \eqref{eq:bar_gamma_pe}, it can be seen that $\bar{\Gamma}_k$ is similar to the standard gain matrix used in recursive least squares \cite{Goodwin_1984} and therefore retains the desired positive-definiteness property with suitable choice of the hyperparameters in \eqref{eq:GPE}. This in turn allowed the proof of Theorem \ref{theo:pe}.
\end{remark}

\begin{remark}
  Theorem \ref{theo:pe} guarantees convergence of $V_k$ to a compact set $D_{PE}$. For static parameters, i.e., $\theta_k^* \equiv \theta^*$ and $\Delta^* = 0$, from \eqref{eq:13} it can be shown that under PE, $\tilde{\theta}_k$ converges exponentially fast towards the origin.
\end{remark}

\begin{remark}
  Note that under assumptions \ref{ass:1} and \ref{ass:2}, \eqref{eq:13} can also be written as
  \begin{equation}
    \label{eq:19}
    V_{PE} = 2\Delta^*\theta_{\max}^*\left[\frac{c_1 + \sqrt{c_1^2 + 4c_2(\mu_1 - \mu_2)}}{2(\mu_1 - \mu_2)}\right]^2.
  \end{equation}
  From \eqref{eq:19}, note that the compact set scales linearly with both the bounds of the unknown parameter $\theta_{\max}^*$ and the time variation of the unknown parameter $\Delta^*$. 
\end{remark}

\begin{remark}
  A few comments about the choice of the hyperparameters in \eqref{eq:3}, \eqref{eq:1} and \eqref{eq:2} are in order. We assume that $\Gamma_{\max}$ is given. Parameter $\lambda_\Omega$ is chosen according to \eqref{eq:lOPE} and $\Omega_{\max}$ is available according to Lemma \ref{lemma:1}. Then $\lambda_\Gamma$ and $\kappa$ will in turn be chosen from the bounds in \eqref{eq:lGPE}, \eqref{eq:kappaBound}. Thus all hyperparameters can be determined and $\Gamma_{PE}$ can be derived from \eqref{eq:GPE}.
\end{remark}
\begin{remark}
  Note that in practice, since we generally do not know the excitation level $\alpha$, the period $\Delta T$ and the maximum normalization signal $d$, it is hard to use \eqref{eq:lOPE} as a guideline to choose $\lambda_\Omega$. Thus careful tuning of the hyperparameters is required.
\end{remark}

\section{Stability and Convergence Analysis for Algorithm 2}
\label{sec:proj}
As mentioned earlier, we assume that the unknown parameter $\theta^*$ is a constant and does not vary with $k$. When there is no persistent excitation in the input signal, the matrix $\Gamma_k$ can exhibit bursting behavior and in fact drift to infinity\cite{Goodwin_1984, Astrom_1995}. To avoid this, and prove stability, we introduce a projection operator in the update of $\Gamma_k$ as a replacement of \eqref{eq:2},
\begin{equation}
	\Gamma_k = \mathrm{Proj}_{\Gamma_{\max}}(\Gamma_k'),
	\label{eq:28}
\end{equation}
where
\begin{equation*}
	\Gamma_k' = \Gamma_{k-1} + \lambda_\Gamma(\Gamma_{k-1} - \kappa\Gamma_{k-1}\Omega_k\Gamma_{k-1}).
\end{equation*}
The function $\mathrm{Proj}_{\Gamma_{\max}}(\cdot)$ is defined as in Algorithm \ref{alg:1}.

The following Lemma establishes boundedness of $\Gamma_k$ in \eqref{eq:28} with no assumptions on PE.
\begin{lemma}
  \label{lemma:7}
  If $\Gamma_{\min}I \preceq \Gamma_0 \preceq \Gamma_{\max}I$ and $\kappa\leq \frac{\lambda_\Gamma}{\lambda_\Omega} [(1+\lambda_\Gamma)\Gamma_{\max}-\Gamma_{\min}]\Gamma_{\max}$, the update of $\Gamma_k$ according to \eqref{eq:28} ensures that $\Gamma_{\min}I \preceq \Gamma_k \preceq \Gamma_{\max}I$ for all $k\geq 0$.
\end{lemma}

\begin{proof}
	Due to the projection in Algorithm \ref{alg:1}, $\Gamma_k \preceq \Gamma_{\max}I$, $\forall k\geq 0$. Since $\Gamma_{k-1}$ is positive definite, $\Gamma_{k-1}$ can be factorized as \eqref{eq:18}, thus
	\begin{align*}
		\Gamma_k &= \mathrm{Proj}_{\Gamma_{\max}}[\Gamma_{k-1} + \lambda_\Gamma(\Gamma_{k-1} - \kappa\Gamma_{k-1}\Omega_k\Gamma_{k-1})]\\
		&\succeq \mathrm{Proj}_{\Gamma_{\max}}[\Gamma_{k-1} + \lambda_\Gamma(\Gamma_{k-1} - \kappa\Omega_{\max}\Gamma_{k-1}^2)]\\
		&= \mathrm{Proj}_{\Gamma_{\max}}\{Q_{k-1}[\Lambda_{k-1}\\
		&\;\;\; + \lambda_\Gamma(\Lambda_{k-1} - \kappa\Omega_{\max}\Lambda_{k-1}^2)]Q_{k-1}^\top\}
	\end{align*}
	Denote the $i$th diagonal element of $\Lambda_{k-1}$ as $\lambda_{k-1,i}$, $i=1,\ldots, N$. The smallest $\lambda_{k,i}$ will be obtained when $\lambda_{k-1,i} = \Gamma_{\max}$. From the upper bound of $\kappa$, $\lambda_{k,i}\geq\Gamma_{\min}$. Using induction, $\Gamma_k \succeq \Gamma_{\min}I$, $\forall k$.
\end{proof}

\begin{remark}
  Lemma \ref{lemma:7} shows that with projection, $\Gamma_k$ will be bounded even without persistent excitation. The projection in Algorithm \ref{alg:1} ensures that the upper bound will always be preserved. The lower bound of $\Gamma_k$ is possible because $\Omega_k$ always has an upper bound.
\end{remark}

Before stating the main stability results, we first define a few quantities. Let
\begin{equation}
  \label{eq:29}
  0 < \lambda_\Omega < 1,
\end{equation}
\begin{equation}
  \label{eq:31}
  \lambda_\Gamma > 0,
\end{equation}
\begin{equation}
  \label{eq:30}
  \Gamma_{\min}I \leq \Gamma_0 \leq \Gamma_{\max}I.
\end{equation}
Parameter $\kappa$ satisfies
\begin{align}
  0 < \kappa < \min&\left\{\frac{1}{\Omega_{\max}\Gamma_{\min}}, \frac{\lambda_\Gamma}{\lambda_\Omega}[(1 + \lambda_\Gamma)\Gamma_{\max} - \Gamma_{\min}]\Gamma_{\max},\right.\nonumber\\
  &\left. \frac{1}{\lambda_\Gamma\Gamma_{\max}}, \frac{1}{(1-\lambda_\Omega)\Gamma_{\max}\Omega_{\max}}\right\},\label{eq:32}
\end{align}
and $\Gamma_{\min}$ is defined as
\begin{equation}
  \label{eq:33}
  \Gamma_{\min} = \min\left\{\Gamma_{\max} + \lambda_\Gamma(\Gamma_{\max}-\kappa\Omega_{\max}\Gamma_{\max}^2), \frac{1}{\kappa\Omega_{\max}}\right\},
\end{equation}
where $\Omega_{\max} = 1/\lambda_\Omega$. Finally, we define
\begin{equation}
  \label{eq:34}
  \bar\Gamma_{\min} = \min\left\{\Gamma_{\max} - \lambda_\Gamma\kappa\Gamma_{\max}^2, \Gamma_{\min} - \lambda_\Gamma \kappa\Gamma_{\min}^2\right\},
\end{equation}
\begin{align}
  \Upsilon_{\min} = \lambda_\Gamma \min &\left\{\Gamma_{\min} - \kappa(1 - \lambda_\Omega)\Gamma_{\min}^2 \Omega_{\max},\right.\nonumber\\
  &\left.\Gamma_{\max} - \kappa(1 - \lambda_\Omega)\Gamma_{\max}^2\Omega_{\max}\right\}.\label{eq:37}
\end{align}
\begin{theorem}
  \label{theo:3}
  For the ARMA model in \eqref{eq:10}, if the parameters in $\theta^*$ are constant, and the hyperparameters $\lambda_\Omega$, $\lambda_\Gamma$ and $\kappa$ satisfy \eqref{eq:29}, \eqref{eq:31} and \eqref{eq:32}, respectively, the adaptive algorithm in \eqref{eq:3}, \eqref{eq:1} and \eqref{eq:28} guarantees that
  \begin{equation*}
    V_k \leq V_{k-1}.
  \end{equation*}
\end{theorem}
\begin{proof}
 From \eqref{eq:3} and \eqref{eq:bar_gamma_pe}, the parameter error can be written as
 \begin{equation}
   \tilde{\theta}_k = \bar{\Gamma}_k\Gamma_{k-1}^{-1}\tilde{\theta}_{k-1},
   \label{eq:35}
 \end{equation}
 where $\tilde{\theta}_k$ is defined as $\tilde{\theta}_k = \theta_k - \theta^*$. Consider again a decomposition of $\Gamma_k$ as in \eqref{eq:rest_v}. Following the approach to proving that $\bar\Gamma_k$ and $\Upsilon_k$ are bounded in \eqref{eq:rest_v}, it is easy to show that $0\prec\bar\Gamma_{\min}I\preceq \bar\Gamma_k \preceq \bar\Gamma_{\max}I$ and $0\prec\Upsilon_{\min}I\preceq \Upsilon_k\preceq \Upsilon_{\max}I$ by considering the hyperparameters in \eqref{eq:29}-\eqref{eq:32}, where $\bar\Gamma_{\min}$, $\Upsilon_{\min}$ are defined in \eqref{eq:34} and \eqref{eq:37}, respectively, and $\bar\Gamma_{\max} = \Gamma_{\max}$, $\Upsilon_{\max} = \lambda_\Gamma\Gamma_{\max}$.
 Substitute \eqref{eq:35} into the Lyapunov function in \eqref{eq:lya} and expand it, we obtain
 \begin{equation}
   V_k = \tilde{\theta}_{k-1}^\top \Gamma_{k-1}^{-1}\tilde{\theta}_{k-1} - \lambda_\Gamma \kappa \frac{\tilde{\theta}_{k-1}^\top\phi_{k-1}\phi_{k-1}^\top\tilde{\theta}_{k-1}}{1 + \|\phi_{k-1}\|^2}.
 \end{equation}
 Since $\bar\Gamma_k\succ 0$ and $\Upsilon_k \succ 0$, the update of $\Gamma_k $ in \eqref{eq:28} indicates that $\Gamma_k \succeq \bar\Gamma_k$. Therefore,
 \begin{align}
   \label{eq:36}
   V_k &\leq \tilde{\theta}_{k-1}^\top \bar\Gamma_{k-1}^{-1}\tilde{\theta}_{k-1} - \lambda_\Gamma \kappa \frac{\tilde{\theta}_{k-1}^\top\phi_{k-1}\phi_{k-1}^\top\tilde{\theta}_{k-1}}{1 + \|\phi_{k-1}\|^2}\\
   &= V_{k-1} - \lambda_\Gamma \kappa \frac{\tilde{\theta}_{k-1}^\top\phi_{k-1}\phi_{k-1}^\top\tilde{\theta}_{k-1}}{1 + \|\phi_{k-1}\|^2}.
 \end{align}
 Thus $\Delta V_k \leq 0$.
\end{proof}
\begin{remark}
The second main contribution of the paper is Theorem \ref{theo:3}. While Algorithm 2 is fairly similar to RLS with forgetting factor, Theorem \ref{theo:3} guarantees that stability can be ensured with Algorithm 2. To our knowledge, RLS with forgetting factor has not been shown to be stable under non-PE conditions. Extension to the case of $\theta_k^*$ is a topic for future work, which will require a combined examination of \eqref{eq:8} and Algorithm \ref{alg:1}.
\end{remark}

\section{Numerical Simulations}
\label{sec:sim}
A discretized and linearized model of F-16 longitudinal dynamics with a sample rate of 10 Hz is given by \cite{Stevens2003},
\begin{equation}
  \label{eq:21}
  G(\bm{q}) = \frac{-0.6213\bm{q} + 0.5839}{\bm{q}^2 - 1.8403\bm{q} + 0.8591}.
\end{equation}
The problem is that the coefficients of \eqref{eq:21} are unknown and we propose the use of Algorithms 1 and 2 to identify them using an input
\begin{equation}
  \label{eq:22}
  u(k) = \sin\left(\frac{3\pi k}{4}\right) + \sin\left(\frac{2\pi k}{5}\right) + \sin\left(\frac{\pi k}{5}\right).
\end{equation}
We compare our algorithms with the standard RLS algorithm
\begin{align}
  \label{eq:23}
  \theta_k &= \theta_{k-1} + \frac{\Gamma_{k-1}\phi_{k-1}e_{k-1}}{1 + \phi_{k-1}^\top \Gamma_{k-2}\phi_{k-1}},\\
  \Gamma_{k-1} &= \Gamma_{k-2} - \frac{\Gamma_{k-2}\phi_{k-1}\phi_{k-1}^\top \Gamma_{k-2}}{1 + \phi_{k-1}^\top \Gamma_{k-2}\phi_{k-1}},\label{eq:24}
\end{align}
and the RLS with forgetting chosen as
\begin{align}
  \label{eq:25}
  \theta_k &= \theta_{k-1} + \frac{\Gamma_{k-1}\phi_{k-1}e_{k-1}}{\alpha + \phi_{k-1}^\top \Gamma_{k-2}\phi_{k-1}},\\
  \Gamma_{k-1} &= \frac{1}{\alpha}\left[\Gamma_{k-2} - \frac{\Gamma_{k-2}\phi_{k-1}\phi_{k-1}^\top \Gamma_{k-2}}{\alpha + \phi_{k-1}^\top \Gamma_{k-2}\phi_{k-1}}\right],\label{eq:26}
\end{align}
where $\alpha = 0.99$. In all simulations, we choose the initial values as $\Gamma_0 = 3I$ and $\theta_0 = 0$.

\subsection{Experiment 1}
\label{sec:experiment-1}
In experiment 1, we choose the following parameters for \eqref{eq:3}-\eqref{eq:2}:
\begin{equation*}
  \lambda_\Omega = 0.1,\; \lambda_\Gamma = 0.4,\; \kappa = 15.
\end{equation*}
Fig. \ref{fig:error_plot} shows the output and parameter error over 100 seconds for Algorithm 1. It can be noted that both RLS with forgetting and Algorithm 1 are faster than the standard RLS algorithm, with Algorithm 1 producing somewhat larger learning rates. Fig. \ref{fig:history_gain} shows the evolution of time varying gain matrices in RLS with forgetting algorithm and our algorithm.
\begin{figure}[htb]
  \minipage{0.5\linewidth}
  \includegraphics[width=\linewidth]{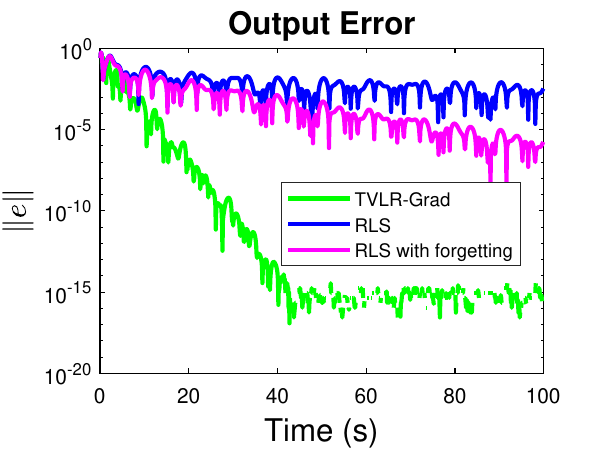}
  \endminipage
  \hfill
  \minipage{0.5\linewidth}
  \includegraphics[width=\linewidth]{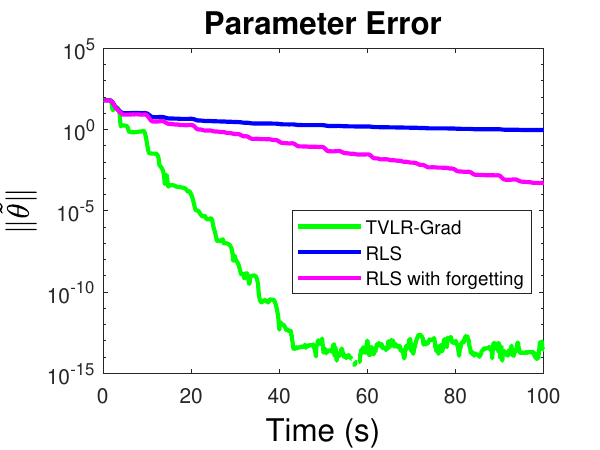}
  \endminipage
  \caption{Experiment 1: Output error and parameter error of the three algorithms.}
  \label{fig:error_plot}
\end{figure}

\begin{figure}[!htb]
  \centering
  \minipage{0.5\linewidth}
  \includegraphics[width=\linewidth]{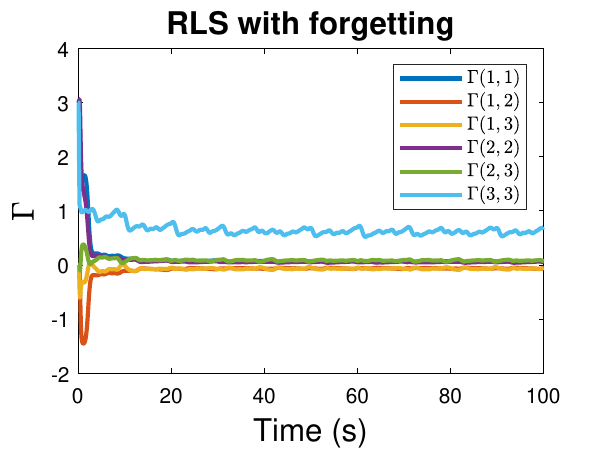}
  \endminipage
  \hfill
  \minipage{0.5\linewidth}
  \includegraphics[width=\linewidth]{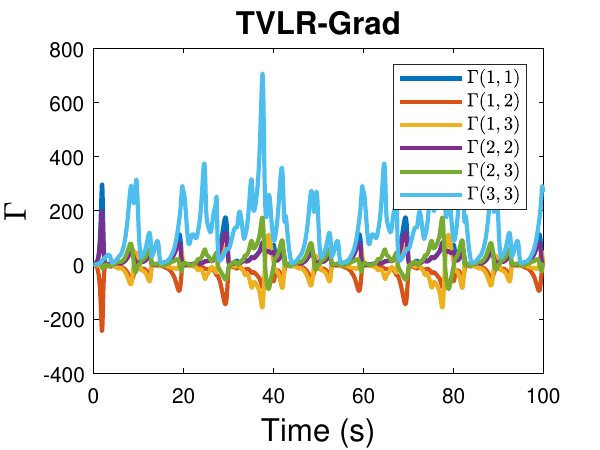}
  \endminipage
  \caption{Experiment 1: Histories of the learning rate matrix.}
  \label{fig:history_gain}
 \end{figure}
 
\subsection{Experiment 2}
\label{sec:experiment-2}
In experiment 2, the following parameters are chosen in Algorithm 2.
\begin{equation*}
  \lambda_\Omega = 0.05,\; \lambda_\Gamma = 0.4,\; \kappa = 15,\; \Gamma_{\max} = 15.
\end{equation*}
Fig. \ref{fig:error2} shows the resulting output error and parameter error. In addition to the obvious effect of a lower learning rate due to the projection, we also note that compared to Fig. \ref{fig:error_plot}, the convergence of our algorithm becomes slower. Fig. \ref{fig:history_gain2} shows the evolution of the time-varying gain matrix in RLS with forgetting and our algorithm. Compared to Fig. \ref{fig:history_gain}, the values in $\Gamma_k$ in our algorithm become a lot smaller and are comparable to those of RLS with forgetting. The advantage of Algorithm 2 over RLS with forgetting, as mentioned in the introduction, is a stability guarantee established above in Theorem \ref{theo:3}.

\begin{figure}[htb]
  \centering
  \minipage{0.5\linewidth}
  \includegraphics[width=\linewidth]{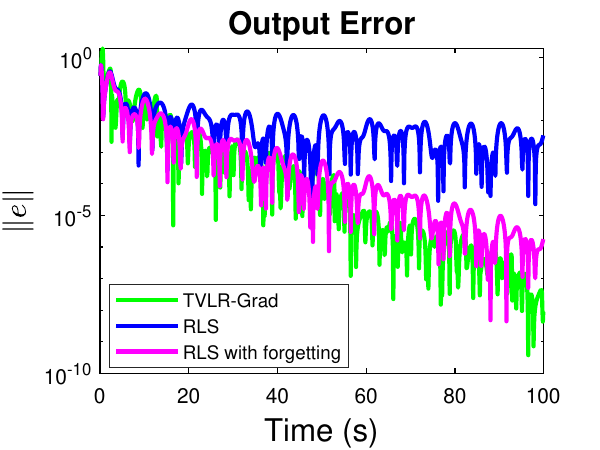}
  \endminipage
  \hfill
  \minipage{0.5\linewidth}
  \includegraphics[width=\linewidth]{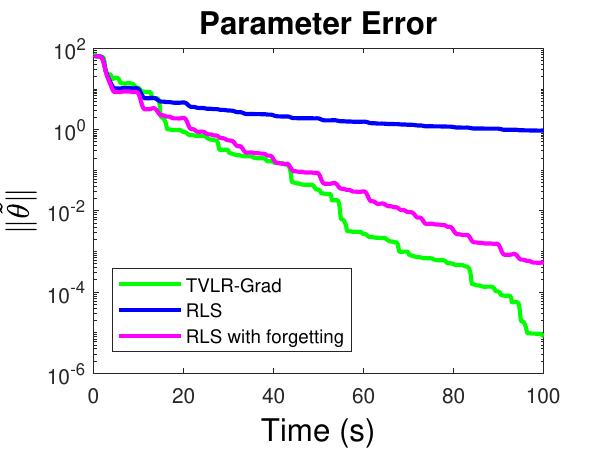}
  \endminipage
  \caption{Experiment 2: Output error and parameter error of the three algorithms.}
  \label{fig:error2}
\end{figure}

\begin{figure}[!htb]
  \centering
  \minipage{0.5\linewidth}
  \includegraphics[width=\linewidth]{notv_rls_forget.pdf}
  \endminipage
  \hfill
  \minipage{0.5\linewidth}
  \includegraphics[width=\linewidth]{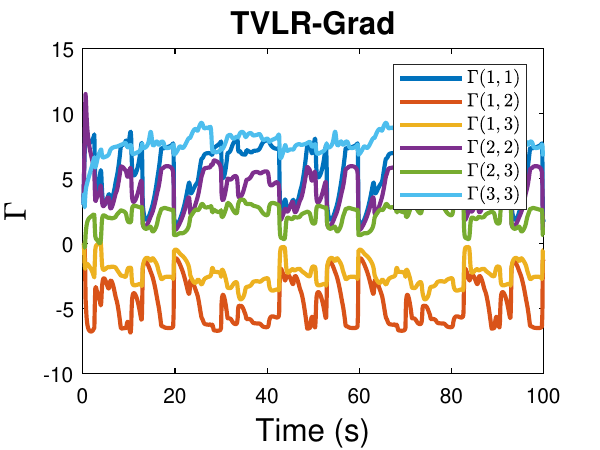}
  \endminipage
  \caption{Experiment 2: Histories of the learning rate matrix.}
  \label{fig:history_gain2}
\end{figure}

\subsection{Experiment 3}
\label{sec:experiment-3}
In experiment 3, we make the underlying true parameters time-varying through multiplying the true parameters by a sinusoidal wave. The time-varying parameters with become
\begin{equation*}
  \theta_k^* = [1 - 0.01 * \sin(0.02k)]\cdot\theta^*,
\end{equation*}
where $\theta^*$ has the same parameters we try to identify in experiment 1 and experiment 2. We adopt the same parameters for Algorithm 2 as in experiment 2. Fig. \ref{fig:error3} shows the output and parameter error. All parameters in the three algorithms converge to a compact set. The performance of Algorithm 2 is comparable to that of RLS with forgetting. Fig. \ref{fig:history_gain3} shows the evolution of the time-varying gain matrix in RLS with forgetting and our Algorithm 2. Again, due to the projection operator, the magnitudes are comparable to those of RLS with forgetting.
\begin{figure}[htb]
	\centering
	\minipage{0.5\linewidth}
	\includegraphics[width=\linewidth]{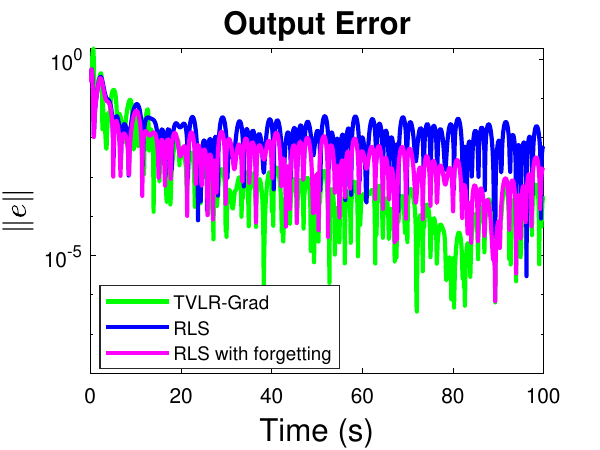}
	\endminipage
	\hfill
	\minipage{0.5\linewidth}
	\includegraphics[width=\linewidth]{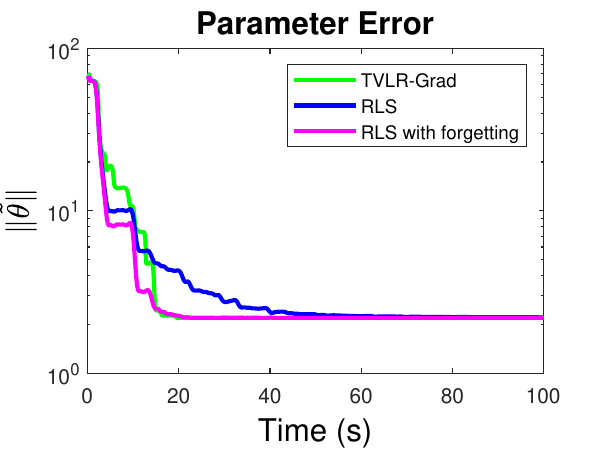}
	\endminipage
	\caption{Experiment 3: Output error and parameter error of the three algorithms.}
	\label{fig:error3}
\end{figure}

\begin{figure}[!htb]
	\centering
	\minipage{0.5\linewidth}
	\includegraphics[width=\linewidth]{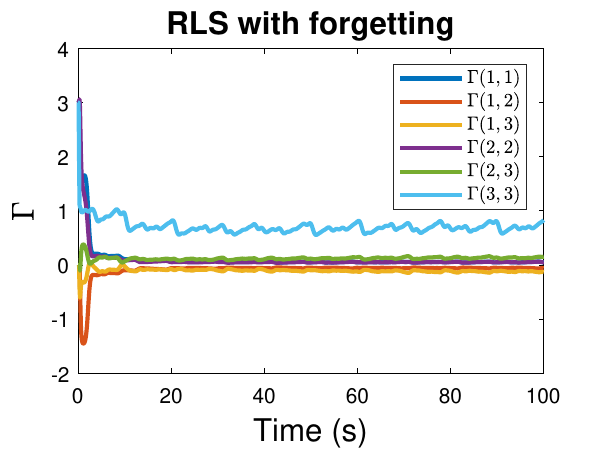}
	\endminipage
	\hfill
	\minipage{0.5\linewidth}
	\includegraphics[width=\linewidth]{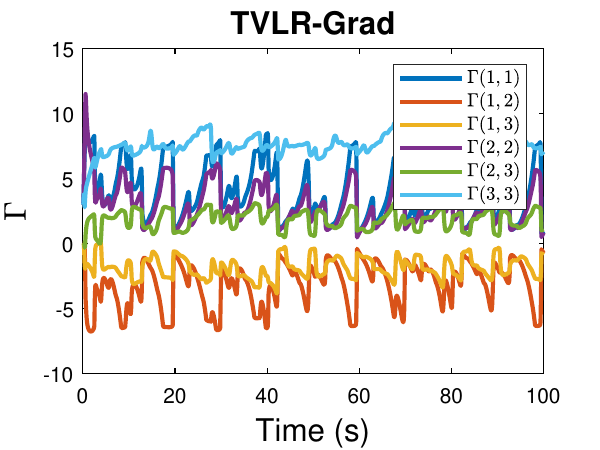}
	\endminipage
	\caption{Experiment 3: Histories of the learning rate matrix.}
	\label{fig:history_gain3}
\end{figure}

\section{CONCLUSIONS}
\label{sec:conclusion}
We have presented new adaptive algorithms, one for discrete-time parameter estimation of a class of time-varying plants when PE conditions are met, and another when no PE conditions are assumed and the plant parameters are constants. We showed that in the presence of time-varying parameters, the parameter estimation error converges uniformly to a compact set under conditions of persistent excitation, with the size of the compact set proportional to the time-variation of the unknown parameters. The bounds for the compact set are shown to scale linearly with both the bounds of the unknown parameter and the time variation of the unknown parameter. The second algorithm leverages a projection operator, which is then shown to result in boundedness guarantees through the use of a Lyapunov function. While this algorithm is similar to RLS with forgetting factor, the important distinction is the proof of stability.


\bibliographystyle{IEEEtran}
\bibliography{IEEEabrv,References.bib}

\end{document}